\documentclass[12pt]{article}
\usepackage[letterpaper,margin=1in]{geometry}
\usepackage[T1]{fontenc}
\usepackage{amsmath,amssymb,amsthm,booktabs,tabularx,array,enumitem,float,tikz,natbib,graphicx,xcolor,xurl}
\usepackage[hidelinks]{hyperref}
\usepackage[nameinlink,noabbrev]{cleveref}
\usetikzlibrary{positioning,arrows.meta}
\definecolor{inktwo}{HTML}{25324A}
\definecolor{private}{HTML}{3C8D5A}
\definecolor{shared}{HTML}{B75D7A}
\definecolor{planner}{HTML}{3977A8}
\definecolor{accent}{HTML}{C98B32}
\newcolumntype{Y}{>{\raggedright\arraybackslash}X}
\newcommand{\E}{\mathbb{E}}
\newcommand{\Prb}{\mathbb{P}}
\renewcommand{\path}[1]{\texttt{\detokenize{#1}}}
\newcommand{\ArtifactNote}[1]{\par\vspace{2pt}{\scriptsize\color{inktwo!72}#1}}
\newcommand{\Evidence}[1]{\par\noindent\textbf{Evidence status.} #1\par}
\newcommand{\Boundary}[1]{\par\noindent\textbf{Boundary.} #1\par}
\newtheorem{theorem}{Theorem}
\newtheorem{proposition}[theorem]{Proposition}
\newtheorem{corollary}[theorem]{Corollary}

\theoremstyle{definition}
\newtheorem{definition}[theorem]{Definition}

\title{\textbf{When Does Information Sharing Improve Decentralized Discovery?}\\[0.35em]\large Aggregation, Independent Rescue, and Equilibrium Selection}
\author{Yohei Nakajima\\Untapped Capital}
\date{July 2026}

\begin{document}
\maketitle
Information sharing changes both posterior quality and the diversity of actions
available for discovery. We study that joint effect in finite one-hit search.
First, exact channel comparisons show that one-person Bayes accuracy does not
determine the value of a pooled action budget. Second, an incremental-sharing
identity separates the accuracy gained by enlarging a pooled block from the
independent rescue action that pooling removes. This yields an exact local
criterion: sharing helps precisely when pooled residual error contracts faster
than the removed private-failure factor. A centralized posterior top-$L$
portfolio can recover the declared private baseline at full action capacity,
but this is an authority benchmark, not a decentralized implementation.

We then study a two-agent equal-split discovery game with a hidden mixture of
common and conditionally independent signal sources. At signal accuracy
$p=3/5$, sharing strictly improves discovery and each symmetric agent's payoff
over the private anonymous symmetric Bayes--Nash selection exactly when
$\rho\in((5\sqrt{73}-17)/48,1)$. This theorem is selection-dependent: the
shared outcome is the declared anonymous posterior-only, provenance-blind,
identical-mixing equilibrium, not an every-equilibrium result. Opposite
constant-target private equilibria exist and attain discovery one, while
ownership-aware symmetric strategies can split targets on public disagreement.
Sharing reveals the agreement pattern and changes posterior beliefs about
dependence; it does not reveal the realized common-versus-independent source
branch. The combined evidence is analytic and exact finite computation. It
uses no human or real data, and the bounded channel registry is not a universal
information order.


\Evidence{The manuscript synthesizes DD-C-0089 through DD-C-0110 without
creating a new scientific claim. Analytic identities and theorems, independently
reproduced bounded computations, a verified bounded negative result, and a
verified equilibrium-selection failure retain their ledger statuses. The source
data are synthetic and model-generated.}
\noindent\textbf{Publication status.} Working paper; not submitted, not peer
reviewed, and no DOI.

\section{Introduction}

Should searchers share what they know before acting? The familiar answer is that
sharing improves a group's estimate of which action is best. That answer is
incomplete when the outcome of interest is discovery by a portfolio of actions.
Sharing can make one common action more accurate while removing independent
attempts that could rescue a mistaken common choice. The organizational question
is therefore not whether the pooled posterior is better. It is whether its gain
is large enough to compensate for the search capacity absorbed by pooling.

This paper develops that comparison in four layers. First, an action-budget
profile records the best centralized discovery probability obtainable from a
pooled signal when the planner may take one, two, or more posterior-ranked
actions. Second, an incremental-sharing identity decomposes a move from $s$ to
$s+1$ pooled agents into the improvement of the pooled block and the loss of one
private rescue action. Third, a residual-error theorem gives the exact frontier:
sharing helps precisely when the pooled error contracts faster than the rescue
technology can repair it. Fourth, a binary Bayesian game shows why a positive
selected-equilibrium result does not settle the decentralized implementation
question.

The central nonstrategic identity is especially transparent. Let $C_s$ be the
probability that the correct candidate is covered by the action budget assigned
to a pooled block of $s$ agents. If each of the remaining $N-s$ independent
searchers succeeds with probability $q$, the registered protocol yields
\begin{equation}
  G_s=1-(1-C_s)(1-q)^{N-s}.
  \label{eq:incremental-value}
\end{equation}
Writing $e_s=1-C_s$, an additional sharing step satisfies
\begin{equation}
 G_{s+1}-G_s=(1-q)^{N-s-1}\big[(1-q)e_s-e_{s+1}\big].
 \label{eq:frontier-intro}
\end{equation}
Thus the step helps if and only if $e_{s+1}/e_s<1-q$, is neutral at equality,
and hurts above the threshold. The statement is exact under the stated rescue
protocol (DD-C-0092 and DD-C-0097). It is not a theorem that more accurate
signals always help, nor is it a comparison of arbitrary information structures.

The binary strategic result supplies a genuine positive sharing interval. Two
agents receive individually accurate clues with $p=3/5$ from a hidden mixture of
a common-source branch and an independent-source branch. The realized branch is
not revealed. Under the registered anonymous, label-equivariant private
equilibrium and the posterior-only identical-mixing shared equilibrium, sharing
strictly improves selected discovery exactly when
\begin{equation}
 \rho\in\left(\frac{5\sqrt{73}-17}{48},1\right).
 \label{eq:positive-interval}
\end{equation}
The same interval improves selected per-agent payoff (DD-C-0106 and
DD-C-0107). This is a coordination-free positive result: no planner assigns
targets and no binding recommendation is added when signals are shared.

The qualification is equally central. The positive interval is
selection-dependent and is not an every-equilibrium result. For every dependence
level, private play also has opposite constant-target pure equilibria that attain
discovery one. After sharing, ownership-aware disagreement strategies can split
targets without being functions only of the common posterior. Those alternatives
show that posterior-only identical mixing is a selection, not the whole strategy
space (DD-C-0109). A centralized top-two planner also obtains discovery one and
strictly dominates the selected shared outcome (DD-C-0110). The positive theorem
and the selection failure are presented together because separating them would
overstate what information sharing implements.

% Generated evidence/architecture asset; do not edit by hand.
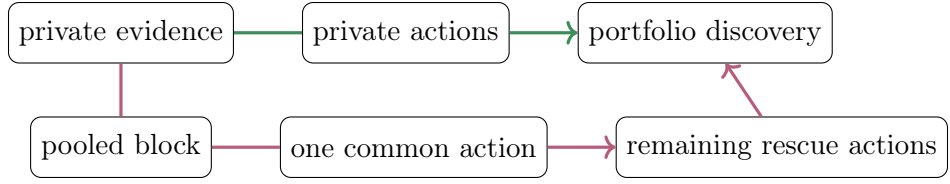
\begin{figure}[t]\centering
\begin{tikzpicture}[node distance=7mm and 9mm, every node/.style={draw,rounded corners,align=center,minimum height=8mm,font=\small}]
\node (e) {private evidence}; \node[right=of e] (p) {private actions}; \node[right=of p] (d) {portfolio discovery};
\node[below=of e] (b) {pooled block}; \node[right=of b] (c) {one common action}; \node[right=of c] (r) {remaining rescue actions};
\draw[->,private,very thick] (e)--(p)--(d);
\draw[->,shared,very thick] (e)--(b)--(c)--(r); \draw[->,shared,very thick] (r)--(d);
\end{tikzpicture}
\caption{Information sharing changes two objects at once. The pooled block can improve one common action, while absorbing agents removes independent rescue actions. The comparison is therefore not a comparison of posterior accuracy alone.}
\label{fig:architecture}
\ArtifactNote{Runs: \path{20260722T142551Z_DD-020_3854fff6_37c11a850a, 20260722T185924Z_DD-021_3cdbbc40_2fea269a9a}; claims: DD-C-0092, DD-C-0097; generator: \path{distributed_discovery.papers.build_information_sharing_frontier}; input SHA-256 prefixes: \texttt{a4f63d52f402, 3521286a3415}.}
\end{figure}

Three further exact findings organize the analysis. One-person accuracy is not a
sufficient statistic for portfolio value: two half-accurate channels can have the
same direct private discovery yet opposite sharing directions and different
centralized recovery budgets (DD-C-0089--DD-C-0091). Full action capacity always
recovers at least the direct private baseline, but the minimum recovery budget
depends on the entire signal geometry (DD-C-0098 and DD-C-0101). Finally, a
complete 177-scenario registry contains compression, aggregation, and all-neutral
curves but no mixed curve. The zero is preserved as a bounded negative result,
not promoted to a general monotonicity theorem (DD-C-0099--DD-C-0103).

The paper contributes a theorem-family synthesis rather than a novelty claim.
The canonical Shared Discovery Paradox provides motivating context and a
conceptual precedent for the atomic information/action separation
\citep{Nakajima2026SharedDiscovery}; it is not a foundational dependency for the
sharing identity used here. That identity and its proof are self-contained in
the incremental-sharing theorem and Appendix~A and retain their DD-020 and
DD-021 authorities regardless of the companion paper's status. The companion
Incentive to Ignore paper owns selective attention and audience design
\citep{Nakajima2026IncentiveToIgnore}.
Decision theory studies comparisons of experiments \citep{Blackwell1953}; team
theory studies decentralized use of information \citep{Radner1962}; social
learning studies herding and cascades \citep{Banerjee1992,BikhchandaniEtAl1992,
AcemogluEtAl2011}; and information design studies the consequences of signal
structures and recommendations \citep{KamenicaGentzkow2011,BergemannMorris2016}.
Our narrower object is the composition of pooled posterior coverage, remaining
rescue actions, and equilibrium selection in finite discovery portfolios. The
literature search and evidence map distributed with the paper explain why broad
priority language would be unsafe.

The rest of the paper proceeds from definitions to geometry, the sharing
frontier, recovery, and strategic implementation. Every displayed numerical
asset is generated from checksum-validated immutable run outputs. Exact numbers
are labeled as fractions; decimals are displays. Centralized benchmarks,
decentralized protocols, selected equilibria, and alternative equilibria are not
interchanged.

\section{Discovery architectures and comparison baselines}

Let $\Theta=\{1,\ldots,M\}$ contain exactly one correct candidate. Discovery
occurs when at least one realized action equals the state. The objective is the
ex ante probability of discovery. A channel maps the state to a signal. Agents
may use signals privately, pool some signals, or submit them to a planner. The
same evidence can therefore enter architectures with different action authority.

\begin{definition}[Direct private baseline]
With $N$ private agents, each agent maps its own signal to one action under the
registered direct rule. For a point signal (including a confidence-tagged
point), the action is its nominated target. For a shortlist, the action is drawn
uniformly from the reported set; for an exclusion signal, it is drawn uniformly
from the non-excluded targets. Equivalently, set-valued signals use uniform
randomization over the registered private Bayes-action correspondence, with
independent private randomization across agents. The resulting discovery
probability is $P_N$.
\end{definition}

\begin{definition}[Centralized action-budget profile]
Given all signals in a pooled block, let
\[
 V_L=\E\left[\sum_{j=1}^{L}\pi_{(j)}\right],
\]
where $\pi_{(1)}\geq\cdots\geq\pi_{(M)}$ are the ordered posterior masses and
$L\leq M$. The recovery budget is
\[
 L^*=\min\{L:V_L\geq P_N\}.
\]
\end{definition}

The definition separates information from capacity. $V_1$ measures the best
single common action after pooling. $V_L$ measures what a planner can cover with
$L$ distinct posterior-ranked actions. Neither number is the discovery of a
decentralized equilibrium unless an institution implements those actions.

% Generated evidence table; do not edit by hand.
\begin{table}[t]\centering\scriptsize
\begin{tabularx}{\textwidth}{Ylll}\toprule channel & signal geometry & $q$ & $P_3$\\\midrule
noisy-point-half & symmetric-noisy-point & $1/2$ & $7/8$\\
noisy-shortlist-three-quarters & noisy-k-shortlist & $3/8$ & $387/512$\\
guaranteed-shortlist-two & guaranteed-shortlist & $1/2$ & $7/8$\\
explicit-exclusion & exclusion & $1/3$ & $19/27$\\
confidence-point & confidence-augmented-point & $5/8$ & $485/512$\\
\bottomrule\end{tabularx}
\caption{Registered channel definitions and named comparison baselines. Accuracy and private discovery are exact; descriptions abbreviate the versioned channel records.}
\label{tab:channels}
\ArtifactNote{Runs: \path{20260722T084145Z_DD-019_a77bb786_04a5e9f0c5}; claims: DD-C-0089; generator: \path{distributed_discovery.papers.build_information_sharing_frontier}; input SHA-256 prefixes: \texttt{a7da8a20aa24}.}
\end{table}

Four comparisons recur. The \emph{private baseline} allows autonomous direct
actions from private signals. The \emph{pooled centralized profile} gives a
planner the pooled posterior and an explicit action budget. The
\emph{incremental-sharing protocol} pools $s$ agents into a common block while
the others retain independent rescue actions. The \emph{strategic protocol}
lets agents maximize a shared discovery prize and requires an equilibrium rule.
These are different counterfactuals, not alternative estimators of one quantity.

The authority distinction matters for interpretation. A planner that chooses
the posterior top two can enforce diversification. Two autonomous agents who
see the same posterior still face a coordination problem: absent roles,
identifiers, or ownership-conditioned strategies, their symmetric response may
duplicate. Conversely, two private agents can coordinate on opposite constant
targets even without using their signals. That behavior is an equilibrium in
the registered binary payoff game but is not direct clue-following. Comparing
the wrong rows can reverse the verbal conclusion without changing any number.

The information structures themselves also differ. A point channel names one
candidate; a shortlist identifies a set; an exclusion signal rules out a
candidate; and a confidence point reports a candidate with a confidence class.
Set-valued prediction is studied as a controlled ambiguity problem
\citep{MortierEtAl2021,SadinleEtAl2019}. Here those geometries are useful because
they can hold one-person accuracy fixed while changing posterior rank mass and
therefore multi-action coverage.

\Boundary{The finite channels are registered synthetic objects. They are not
estimated representations of laboratories, firms, committees, or human teams.
The analysis compares discovery architectures within those objects.}

\section{Signal geometry is not one-person accuracy}

The first result is a sufficiency failure. One-person accuracy $q$ records the
probability that a direct action is correct. It discards how posterior error is
distributed across the remaining candidates. Yet a second or third action is
valuable precisely because of that distribution.

\begin{proposition}[Action-budget profiles; DD-C-0089]
In the frozen $M=4,N=3$ registry, the five channel families have the exact
profiles reported in \cref{tab:profiles}. In particular, the half-accurate noisy
point and guaranteed-shortlist channels both have $q=1/2$ and direct private
discovery $P_3=7/8$, but their centralized profiles are respectively
\[
 (V_1,V_2,V_3)=\left(\frac{7}{12},\frac{43}{54},\frac{25}{27}\right)
 \quad\text{and}\quad
 \left(\frac{17}{18},1,1\right).
\]
\end{proposition}

The computations are finite exact enumerations independently reproduced from
the registered model and fixtures. The proposition does not rank all point and
shortlist channels. It exhibits two channels for which equal accuracy and equal
direct private discovery fail to determine pooled portfolio value.

% Generated evidence asset; do not edit by hand.
\begin{figure}[t]\centering
\begin{tikzpicture}[x=1.75cm,y=4.6cm]
\draw[->] (0.65,0)--(4.45,0); \draw[->] (0.75,0)--(0.75,1.08) node[above] {probability};
\draw[very thick,private,mark=*] plot coordinates {(1,0.583333) (2,0.796296) (3,0.925926)};
\draw[very thick,planner,mark=*] plot coordinates {(1,0.944444) (2,1.000000) (3,1.000000)};
\foreach \x/\lab in {1/$V_1$,2/$V_2$,3/$V_3$}{\draw (\x,0)--(\x,-.018) node[below] {\lab};}
\draw[dashed] (.75,.5)--(3.15,.5) node[right] {one-person accuracy $=1/2$};
\node[private,anchor=west] at (1.1,.60) {noisy point}; \node[planner,anchor=west] at (1.1,.98) {guaranteed shortlist};
\end{tikzpicture}
\caption{Same one-person accuracy, different pooled action-budget profiles. Both channels also have direct private discovery $7/8$, yet the noisy point requires three pooled actions to recover that baseline and the guaranteed shortlist requires one. Exact values are in \cref{tab:profiles}.}
\label{fig:same-accuracy}
\ArtifactNote{Runs: \path{20260722T084145Z_DD-019_a77bb786_04a5e9f0c5}; claims: DD-C-0089, DD-C-0090, DD-C-0091; generator: \path{distributed_discovery.papers.build_information_sharing_frontier}; input SHA-256 prefixes: \texttt{a7da8a20aa24}.}
\end{figure}
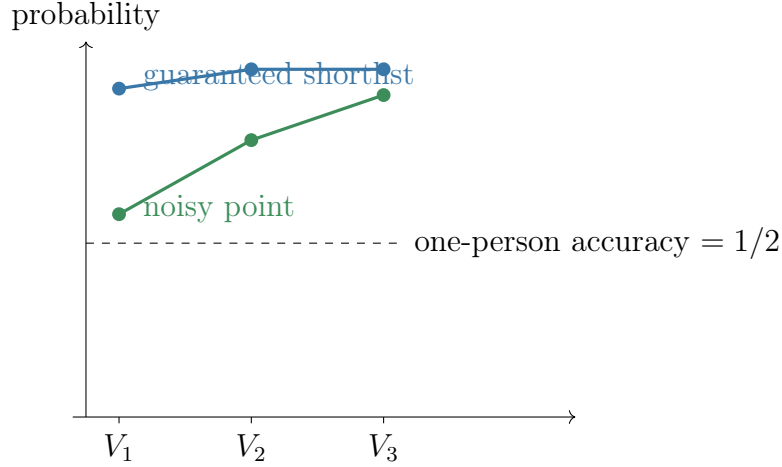

% Generated evidence table; do not edit by hand.
\begin{table}[t]\centering\scriptsize
\begin{tabularx}{\textwidth}{Ylll}\toprule channel & $(V_1,V_2,V_3)$ & decimals & $L^*$\\\midrule
noisy-point-half & $(7/12,43/54,25/27)$ & (0.583, 0.796, 0.926) & 3\\
noisy-shortlist-three-quarters & $(35/64,79/96,539/576)$ & (0.547, 0.823, 0.936) & 2\\
guaranteed-shortlist-two & $(17/18,1,1)$ & (0.944, 1.000, 1.000) & 1\\
explicit-exclusion & $(16/27,26/27,1)$ & (0.593, 0.963, 1.000) & 2\\
confidence-point & $(295/384,175/192,31/32)$ & (0.768, 0.911, 0.969) & 3\\
\bottomrule\end{tabularx}
\caption{Exact DD-019 action-budget profiles. Decimals aid reading; exact fractions govern comparisons and recovery budgets.}
\label{tab:profiles}
\ArtifactNote{Runs: \path{20260722T084145Z_DD-019_a77bb786_04a5e9f0c5}; claims: DD-C-0089, DD-C-0091; generator: \path{distributed_discovery.papers.build_information_sharing_frontier}; input SHA-256 prefixes: \texttt{a7da8a20aa24}.}
\end{table}

\begin{corollary}[Accuracy is insufficient; DD-C-0090]
Neither one-person accuracy nor the pair $(q,P_3)$ determines the centralized
action-budget profile or the minimum recovery budget.
\end{corollary}

For the noisy point, one pooled action covers only $7/12$ and two cover $43/54$;
three are required to recover $7/8$. For the guaranteed shortlist, one pooled
action already covers $17/18$, so $L^*=1$. The exclusion and confidence-point
channels provide further distinct shapes: the former reaches certainty at three
actions, while the latter has high direct private discovery $485/512$ but needs
three pooled actions to recover it.

\begin{proposition}[Registered recovery budgets; DD-C-0091]
For noisy point, noisy shortlist, guaranteed shortlist, exclusion, and confidence
point, the recovery budgets are respectively $3,2,1,2,3$.
\end{proposition}

The result suggests a practical diagnostic. An organization considering pooled
evaluation should estimate or stress-test a coverage curve, not report only the
top-choice accuracy. If pooled evidence concentrates almost all posterior mass
on a small set, a modest common action budget can be enough. If its errors remain
diffuse or systematically compress plausible alternatives behind one label,
pooling may require many actions to restore the private portfolio baseline.

This diagnostic is related to, but not supplied by, Blackwell's comparison of
experiments. Blackwell dominance ranks experiments for every decision problem
with a single decision maker \citep{Blackwell1953}. The present comparisons hold
the discovery problem fixed while changing how many actions remain and who
controls them. A scalar accuracy ordering is weaker still. We therefore do not
describe the bounded registry as a universal information order.

\section{Aggregation gain and independent rescue}

Consider $N$ searchers. A pooled block of size $s$ produces coverage $C_s$ from
its assigned action budget. Each of the $N-s$ remaining private searchers has an
independent probability $q$ of rescuing a miss. Conditional on a pooled miss,
all private rescue actions fail with probability $(1-q)^{N-s}$. This yields
\cref{eq:incremental-value}.

\begin{theorem}[Incremental-sharing identity; DD-C-0092]
For $1\leq s<N$,
\[
 G_{s+1}-G_s=(1-q)^{N-s-1}\{(1-q)(1-C_s)-(1-C_{s+1})\}.
\]
Consequently the sign of an additional sharing step is the sign of the pooled
coverage gain net of the private rescue action it replaces.
\end{theorem}

\begin{proof}
Subtract $1-(1-C_s)(1-q)^{N-s}$ from
$1-(1-C_{s+1})(1-q)^{N-s-1}$ and factor the positive term
$(1-q)^{N-s-1}$. No distributional approximation is used.
\end{proof}

\begin{corollary}[Noisy point monotonicity; DD-C-0093]
For the registered noisy point family, the coverage improvement from adding a
signal does not overcome the lost rescue action; the incremental-sharing curve
is weakly decreasing on the proved parameter domain.
\end{corollary}

The corollary is channel-specific. It does not follow from $q=1/2$. The
registered guaranteed-shortlist channel has the same accuracy and moves in the
opposite direction.

% Generated evidence asset; do not edit by hand.
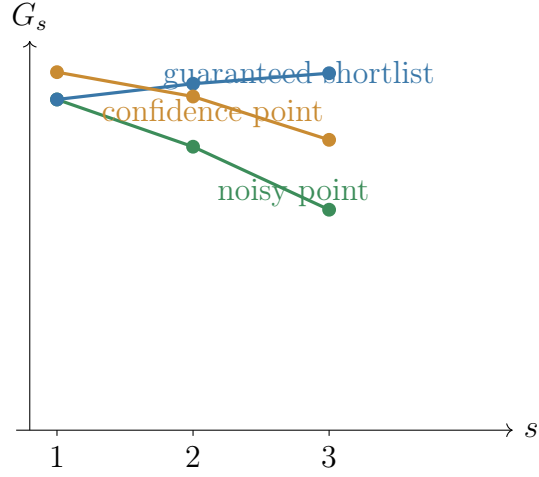
\begin{figure}[t]\centering
\begin{tikzpicture}[x=1.8cm,y=5cm]
\draw[->] (.7,0)--(4.35,0) node[right] {$s$}; \draw[->] (.8,0)--(.8,1.03) node[above] {$G_s$};
\draw[very thick,private,mark=*] plot coordinates {(1,0.875000) (2,0.750000) (3,0.583333)};
\draw[very thick,planner,mark=*] plot coordinates {(1,0.875000) (2,0.916667) (3,0.944444)};
\draw[very thick,accent,mark=*] plot coordinates {(1,0.947266) (2,0.882812) (3,0.768229)};
\foreach \x in {1,2,3}{\draw (\x,0)--(\x,-.015) node[below] {\x};}
\node[private,anchor=west] at (2.1,.63) {noisy point};
\node[planner,anchor=west] at (1.7,.94) {guaranteed shortlist};
\node[accent,anchor=west] at (1.25,.84) {confidence point};
\end{tikzpicture}
\caption{Incremental-sharing curves at $M=4,N=3$. The half-accurate noisy point declines, the equally accurate guaranteed shortlist rises, and the confidence-point channel declines. The sign is a property of channel geometry and protocol, not accuracy alone.}
\label{fig:incremental}
\ArtifactNote{Runs: \path{20260722T142551Z_DD-020_3854fff6_37c11a850a}; claims: DD-C-0094, DD-C-0096; generator: \path{distributed_discovery.papers.build_information_sharing_frontier}; input SHA-256 prefixes: \texttt{a4f63d52f402}.}
\end{figure}

% Generated evidence table; do not edit by hand.
\begin{table}[t]\centering\scriptsize
\begin{tabularx}{\textwidth}{Yllll}\toprule channel & $C_2$ & $(G_1,G_2,G_3)$ & adjacent changes & sign\\\midrule
noisy-point-half & $1/2$ & $(7/8,3/4,7/12)$ & $(-1/8,-1/6)$ & decreasing\\
noisy-shortlist-three-quarters & $1/2$ & $(387/512,11/16,35/64)$ & $(-35/512,-9/64)$ & decreasing\\
guaranteed-shortlist-two & $5/6$ & $(7/8,11/12,17/18)$ & $(1/24,1/36)$ & increasing\\
explicit-exclusion & $4/9$ & $(19/27,17/27,16/27)$ & $(-2/27,-1/27)$ & decreasing\\
confidence-point & $11/16$ & $(485/512,113/128,295/384)$ & $(-33/512,-11/96)$ & decreasing\\
\bottomrule\end{tabularx}
\caption{Exact DD-020 sharing paths and intermediate two-signal pooled coverage. The pooled block grows from $s=1$ to $3$ while the remaining private rescue actions disappear.}
\label{tab:sharing-paths}
\ArtifactNote{Runs: \path{20260722T142551Z_DD-020_3854fff6_37c11a850a}; claims: DD-C-0092, DD-C-0096; generator: \path{distributed_discovery.papers.build_information_sharing_frontier}; input SHA-256 prefixes: \texttt{a4f63d52f402}.}
\end{table}

\begin{proposition}[Opposite registered paths; DD-C-0094]
At $M=4,N=3,q=1/2$, the noisy-point path is
$7/8,3/4,7/12$, with changes $-1/8,-1/6$, whereas the
guaranteed-shortlist path is $7/8,11/12,17/18$, with changes
$1/24,1/36$.
\end{proposition}

The DD-020 census has 2,555 protocol rows arranged into 511 parameter chains,
each containing one row for every sharing-block size. Each chain contributes
one initial row, so 2,555 rows minus 511 chain starts yields 2,044 adjacent
transitions. The chain starts are not unclassified observations: the adjacent
transitions are exactly 1,848
negative, 196 neutral, and zero positive (DD-C-0095). The registered grid is
$M=2,\ldots,8$, $N=2,\ldots,8$, every declared accuracy in the 73 $(M,p)$
cells, and all $s=1,\ldots,N$. Those counts reproduce the theorem's sign on a
bounded grid; they are not the proof.
The guaranteed-shortlist row is a verified counterchannel to any claim that
sharing must decrease simply because the private portfolio loses independent
actions (DD-C-0096).

The decomposition clarifies several verbal confusions. Sharing does not merely
``add information.'' It expands the pooled information set while changing the
portfolio's action structure. Nor does diversification merely ``waste'' the
pooled posterior: a low-ranked private action can be valuable precisely on the
event that the pooled block is wrong. The sign is controlled by the product of
these two margins.

\section{The General Sharing Frontier}

The identity becomes a general criterion after expressing coverage in residual
errors. Let $e_s=1-C_s$. When $e_s>0$, define the contraction ratio
$\rho_s=e_{s+1}/e_s$. The private rescue threshold is $1-q$.

\begin{theorem}[General Sharing Frontier; DD-C-0097]
Under the registered independent-rescue protocol with $q<1$,
\[
 \operatorname{sign}(G_{s+1}-G_s)
 =\operatorname{sign}\big((1-q)e_s-e_{s+1}\big).
\]
If $e_s>0$, the step strictly improves discovery exactly when
$\rho_s<1-q$, is neutral exactly when $\rho_s=1-q$, and strictly reduces
discovery exactly when $\rho_s>1-q$.
\end{theorem}

\begin{proof}
The first statement is \cref{eq:frontier-intro}; its prefactor is positive for
$q<1$. Dividing the bracket by $e_s>0$ gives the ratio criterion. If $e_s=0$,
pooled coverage is already one and the undivided expression governs the boundary.
\end{proof}

The theorem is a break-even rule. The new pooled signal must shrink residual
error by more than an independent private attempt would. A larger $q$ makes
rescue more powerful and lowers $1-q$, so pooling faces a stricter contraction
test. A weaker rescue technology raises the threshold and makes pooling easier
to justify. The result is deliberately architectural: changing dependence among
rescue actions or allowing the pooled block to control a different number of
actions changes the formula.

% Generated evidence asset; do not edit by hand.
\begin{figure}[t]\centering
\begin{tikzpicture}[x=2.1cm,y=4.6cm]
\draw[->] (.7,0)--(3.55,0) node[right] {step $s$}; \draw[->] (.8,0)--(.8,1.05) node[above] {$\rho_s$};
\draw[dashed,very thick] (.8,.5)--(3.15,.5) node[right] {$1-q=1/2$};
\draw[very thick,private,mark=*] plot coordinates {(1,1.000000) (2,0.833333)};
\draw[very thick,planner,mark=*] plot coordinates {(1,0.333333) (2,0.333333)};
\foreach \x in {1,2}{\draw (\x,0)--(\x,-.015) node[below] {\x};}
\node[private,anchor=west] at (1.1,.84) {compression: $\rho_s>1-q$};
\node[planner,anchor=west] at (1.1,.18) {aggregation: $\rho_s<1-q$};
\end{tikzpicture}
\caption{Residual-error frontier for the registered half-accurate $M=4,N=3$ point and guaranteed-shortlist channels. A sharing step helps below the rescue threshold, is neutral on it, and hurts above it.}
\label{fig:error-frontier}
\ArtifactNote{Runs: \path{20260722T185924Z_DD-021_3cdbbc40_2fea269a9a}; claims: DD-C-0097, DD-C-0100; generator: \path{distributed_discovery.papers.build_information_sharing_frontier}; input SHA-256 prefixes: \texttt{3521286a3415}.}
\end{figure}
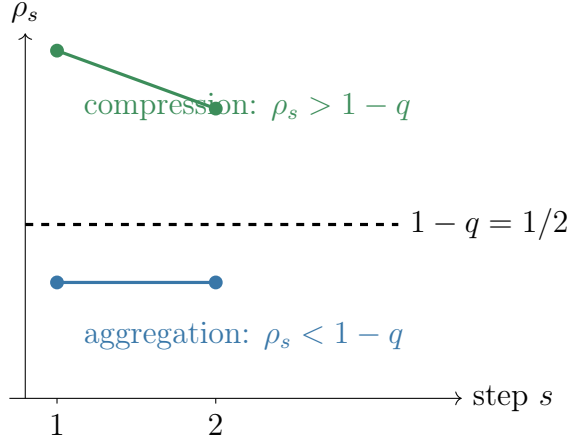

The frontier reconciles apparently contradictory examples. The noisy point
compresses error too slowly relative to a half-accurate rescue action. The
guaranteed shortlist eliminates enough residual mass to compensate for losing
that action. Equal $q$ therefore places both examples against the same horizontal
threshold while their error-contraction curves fall on opposite sides.

\begin{proposition}[Bounded frontier classification; DD-C-0099 and DD-C-0100]
The frozen 177-scenario registry contains 126 strict-compression-dominated
curves, 16 strict-aggregation-dominated curves, and 35 all-neutral curves. The
stored error ratios reproduce the theorem's sign in every registered row.
\end{proposition}

% Generated evidence asset; do not edit by hand.
\begin{figure}[t]\centering
\begin{tikzpicture}[x=2.05cm,y=.045cm]
\draw[->] (.55,0)--(4.7,0); \draw[->] (.65,0)--(.65,140) node[above] {scenarios};
\fill[private!70] (1-.28,0) rectangle (1+.28,126); \node[above] at (1,126) {126}; \node[below,font=\scriptsize] at (1,0) {compression};
\fill[planner!70] (2-.28,0) rectangle (2+.28,16); \node[above] at (2,16) {16}; \node[below,font=\scriptsize] at (2,0) {aggregation};
\fill[inktwo!70] (3-.28,0) rectangle (3+.28,35); \node[above] at (3,35) {35}; \node[below,font=\scriptsize] at (3,0) {neutral};
\fill[accent!70] (4-.28,0) rectangle (4+.28,0); \node[above] at (4,0) {0}; \node[below,font=\scriptsize] at (4,0) {mixed};
\end{tikzpicture}
\caption{Complete bounded classification of 177 registered scenarios: 126 strict compression curves, 16 strict aggregation curves, 35 all-neutral curves, and no mixed curve. The zero mixed count is a bounded null, not a general theorem.}
\label{fig:registry}
\ArtifactNote{Runs: \path{20260722T185924Z_DD-021_3cdbbc40_2fea269a9a}; claims: DD-C-0099, DD-C-0103; generator: \path{distributed_discovery.papers.build_information_sharing_frontier}; input SHA-256 prefixes: \texttt{2bd510abf8cd}.}
\end{figure}
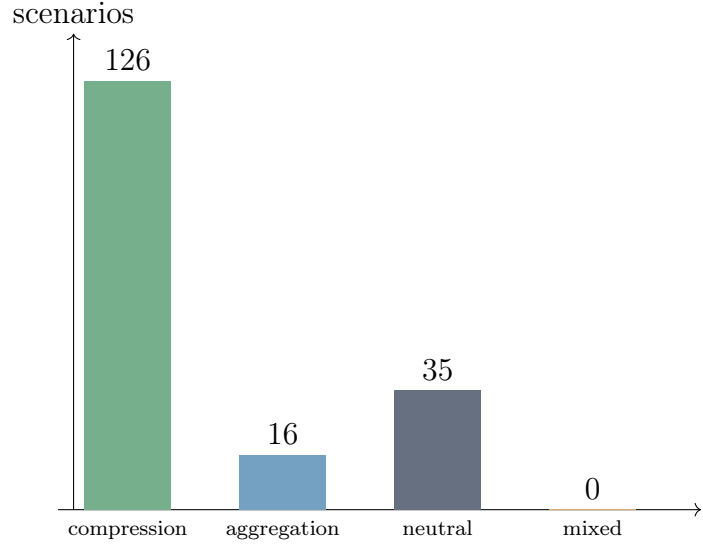

% Generated evidence table; do not edit by hand.
\begin{table}[t]\centering\scriptsize
\begin{tabularx}{\textwidth}{Ylll}\toprule axis & class & count & status\\\midrule
sharing curve & all-neutral & 35 & exact bounded\\
sharing curve & strict-aggregation-dominated & 16 & exact bounded\\
sharing curve & strict-compression-dominated & 126 & exact bounded\\
full sharing & B-shared-discovery-paradox & 78 & exact bounded\\
full sharing & C-strict-aggregation-dominated-consensus & 16 & exact bounded\\
full sharing & D-boundary & 83 & exact bounded\\
recovery budget & $L^*=1$ & 51 & centralized\\
recovery budget & $L^*=2$ & 55 & centralized\\
recovery budget & $L^*=3$ & 64 & centralized\\
recovery budget & $L^*=4$ & 7 & centralized\\
sharing curve & mixed & 0 & bounded null\\
\bottomrule\end{tabularx}
\caption{Complete bounded DD-021 classifications and recovery budgets. Counts describe the frozen registry, not a population of channels or organizations.}
\label{tab:registry-counts}
\ArtifactNote{Runs: \path{20260722T185924Z_DD-021_3cdbbc40_2fea269a9a}; claims: DD-C-0099, DD-C-0101, DD-C-0102, DD-C-0103; generator: \path{distributed_discovery.papers.build_information_sharing_frontier}; input SHA-256 prefixes: \texttt{2bd510abf8cd}.}
\end{table}

The names describe the frozen curves. ``Compression dominated'' means residual
error did not contract fast enough relative to rescue at each applicable step;
``aggregation dominated'' means it contracted faster; ``all neutral'' means the
frontier held with equality. The registry is a systematic stress test across
named finite channel families and parameter values, not a probability sample of
real organizations.

For the full-sharing endpoint, the pooled block has $s=N$ and its one-action
consensus coverage is $C_N$. The full-sharing labels in
\cref{tab:registry-counts} are generated directly from that endpoint: class B
means $q<C_N<P_N$ (a shared-discovery paradox), class C means $C_N>P_N$
(strict aggregation-dominated consensus), and class D is the remaining
equality boundary after those strict cases (so $C_N=q$ or $C_N=P_N$). In the
frozen registry, B comprises 78 of the 126 strict-compression-dominated curves;
C comprises all 16 strict-aggregation-dominated curves; and D comprises the 35
all-neutral curves plus the remaining 48 strict-compression-dominated curves.
These are registry cross-tabs, not additional frontier theorems.

DD-021 and DD-020 deliberately use different scopes. DD-021 enumerates five
registered channel families over $M=3,4,5$, $N=2,3,4$, and
$L=1,\ldots,\min(N,M)$: 59 channel laws produce 177 scenarios for the general
frontier and centralized recovery classification. DD-020 instead isolates the
symmetric noisy-point monotonicity theorem on the broader $M=2,\ldots,8$,
$N=2,\ldots,8$ declared accuracy grid, while carrying the five fixed DD-019
channels only as named comparison paths. The larger DD-020 row count tests one
channel theorem across more parameters; the smaller DD-021 registry compares
more channel geometries and recovery budgets on a deliberately bounded grid.

\begin{proposition}[Minimal registered witnesses; DD-C-0100]
Under the frozen lexicographic ordering, the registry contains minimal witnesses
for equal-baseline opposite signs, equal-accuracy recovery differences, a
shared-discovery-paradox configuration, and consensus dominance. Their exact
values appear in \cref{tab:witnesses}.
\end{proposition}

% Generated evidence table; do not edit by hand.
\begin{table}[t]\centering\scriptsize
\begin{tabularx}{\textwidth}{Ylll}\toprule witness & left or focal value & right or comparator & conclusion\\\midrule
same baseline, opposite signs & point: $-1/4$ & guaranteed shortlist: $1/12$ & $M=4,N=2,q=1/2$\\
same accuracy, recovery & point: $L^*=2$ & guaranteed shortlist: $L^*=1$ & $M=3,N=2,q=1/2$\\
Shared Discovery Paradox & $q=3/8$ & $C_N=27/64<P_N=39/64$ & noisy shortlist\\
consensus dominance & $P_N=3/4$ & $C_N=5/6$ & guaranteed shortlist\\
\bottomrule\end{tabularx}
\caption{Minimal registered witnesses under the frozen lexicographic order. Minimality is internal to the registry.}
\label{tab:witnesses}
\ArtifactNote{Runs: \path{20260722T185924Z_DD-021_3cdbbc40_2fea269a9a}; claims: DD-C-0100; generator: \path{distributed_discovery.papers.build_information_sharing_frontier}; input SHA-256 prefixes: \texttt{e897c62593b4}.}
\end{table}

\begin{proposition}[Preserved bounded null; DD-C-0103]
No registered curve is mixed: none crosses from aggregation dominated to
compression dominated, or conversely, over its recorded sharing steps.
\end{proposition}

\Boundary{The zero mixed count is a verified bounded negative result. It does
not imply that arbitrary channels have monotone sharing curves. A mixed curve
outside the registry would not contradict the theorem; it would simply place
different steps on different sides of the exact frontier. The bounded zero also
does not establish whether mixed curves are feasible within each registered
family's unrestricted parameterization; that structural question was not tested.}

This distinction illustrates the paper's evidence discipline. The analytic
criterion applies wherever its assumptions hold. The class counts apply only to
the frozen registry. The absence of a class is retained because negative results
inform the next search, but it is not converted into a universal statement.

\section{Centralized action-budget recovery}

Pooling need not be evaluated at one common action. If a planner can assign
several actions after observing pooled evidence, the top-$L$ posterior coverage
$V_L$ is nondecreasing in $L$. At full candidate capacity it is one. The direct
private portfolio cannot exceed one, giving a general recovery statement.

\begin{theorem}[Full-capacity recovery; DD-C-0098]
For every registered finite channel with $M$ candidates,
\[
 V_M=1\geq P_N.
\]
Hence a finite minimum recovery budget $L^*\leq M$ exists. Moreover,
$L^*=1$ if and only if $C_N=V_1\geq P_N$.
\end{theorem}

\begin{proof}
The posterior masses sum to one, so choosing all $M$ candidates covers the state
with probability one. Monotonicity of ordered partial sums implies existence of
the first $L$ reaching $P_N$. The last statement is the definition at $L=1$.
\end{proof}

The theorem is an existence result under centralized authority, not a claim
that $M$ decentralized agents will choose distinct actions. It also does not say
that pooling is costless: an organization may lack $M$ actions or may value
actions for other tasks. It identifies the precise capacity question that must
be answered after pooling.

\begin{proposition}[Bounded recovery census; DD-C-0101 and DD-C-0102]
In the frozen 177-scenario registry, recovery budgets $L^*=1,2,3,4$ occur in
$51,55,64,7$ scenarios respectively. Every registered scenario satisfies the
full-capacity theorem and the exact stored recovery characterization.
\end{proposition}

The distribution demonstrates why a binary recommendation---share or do not
share---is usually underspecified. For 51 scenarios, one centralized action
after pooling already matches the direct private baseline. For 126, at least two
actions are required, and seven need all four. The relevant design variable is
the pair consisting of information architecture and post-sharing action budget.

The distinction also helps compare centralization with communication. A team
can share signals without granting a planner binding authority; it can grant a
planner an action budget without pooling all raw signals; and it can retain
private rescue capacity after producing a common recommendation. These designs
occupy different points in the space summarized by \cref{fig:architecture}.

\section{Coordination-free positive sharing}

We now turn from nonstrategic rescue to a Bayesian game. There are two agents
and two target labels. The state is equally likely. Each clue is correct with
probability $p>1/2$. A hidden branch is common with probability $\rho$ and
independent with probability $1-\rho$. In the common branch the clue-generation
shock is shared; in the independent branch it is not. Agents do not observe the
realized branch. After actions, discovery pays one if at least one chosen label
is correct, split equally among agents choosing that correct label.

Let $t=2p-1$ and
\[
 A=t^2+\rho(1-t^2).
\]
Here $A$ is the unconditional clue-alignment statistic: $t^2$ is its
accuracy-driven component and $\rho(1-t^2)$ is its source-dependence component.
The source-dependence mixture probability $\rho$ is distinct from the
residual-error contraction ratio $\rho_s$ in the General Sharing Frontier.
The private selection is the registered anonymous, label-equivariant Bayesian
Nash equilibrium. The shared selection uses identical mixing as a function of
the common posterior only. These restrictions make the compared objects exact
and reproducible, but they are selection assumptions.

\begin{theorem}[Selected equilibrium formulas; DD-C-0104 and DD-C-0105]
In the private selection, the probability $r$ of following one's clue is
\[
 r^*=\begin{cases}
 1,&A\leq3t,\\[2pt]
 \frac12+\frac{3t}{2A},&A>3t.
 \end{cases}
\]
After sharing, agreement produces posterior correctness
\[
 u=\frac12+\frac{t}{1+A}.
\]
The selected probability $x$ of following the agreed label is
\[
 x^*=\begin{cases}
 1,&u\geq2/3,\\
 3u-1,&u<2/3,
 \end{cases}
\]
while disagreement uses the symmetric probability $1/2$.
\end{theorem}

The formulas follow from indifference and boundary conditions in the two-action
game. They do not assume that sharing reveals which dependence branch generated
the clues. Agreement is informative about the state and about the likelihood of
common generation; it does not disclose the realized common/independent branch.

% Generated evidence table; do not edit by hand.
\begin{table}[t]\centering\scriptsize
\begin{tabularx}{\textwidth}{Ylll}\toprule information & state statistic & selected rule & scope\\\midrule
private & $t=2p-1$, $A=t^2+\rho(1-t^2)$ & $r^*=1$ if $A\leq3t$; else $1/2+3t/(2A)$ & anonymous label-equivariant BNE\\
shared agreement & $u=1/2+t/(1+A)$ & $x^*=1$ if $u\geq2/3$; else $3u-1$ & posterior-only identical mixing\\
shared disagreement & posterior $1/2$ & $x^*=1/2$ & ownership-blind selection\\
\bottomrule\end{tabularx}
\caption{Exact DD-022 selected-equilibrium formulas. The selections are deliberately narrower than the complete pure correspondence.}
\label{tab:equilibrium-formulas}
\ArtifactNote{Runs: \path{20260722T210334Z_DD-022_2376d5b7_ad67765ca8}; claims: DD-C-0104, DD-C-0105; generator: \path{distributed_discovery.papers.build_information_sharing_frontier}; input SHA-256 prefixes: \texttt{98fb7d475365}.}
\end{table}

\begin{theorem}[Strict selected sharing gain at $p=3/5$; DD-C-0106]
Let $p=3/5$. Shared selected discovery is strictly greater than private selected
discovery exactly for
\[
 \rho\in(\rho^*,1),\qquad
 \rho^*=\frac{5\sqrt{73}-17}{48}.
\]
Equality holds at $\rho^*$ and at $1$. The certified isolating interval is
$2679/5000<\rho^*<67/125$.
\end{theorem}

\begin{proof}[Proof sketch]
Substitute $t=1/5$ into the regime formulas. The shared rule changes regime at
$\rho=1/6$ and the private rule at $\rho=7/12$. On each interval the discovery
difference is rational in $\rho$. Clearing positive denominators yields the
certified algebraic equality with its unique interior root
$(5\sqrt{73}-17)/48$. Exact sign checks on the isolated intervals give a strict
gain between that root and one, with equality restored at perfect dependence.
The immutable certificate supplies the symbolic factors and rational isolating
interval; no floating-point root is used in the theorem.
\end{proof}

\begin{corollary}[Selected payoff gain; DD-C-0107]
At $p=3/5$, the same open interval strictly improves expected payoff per agent
under the registered selections.
\end{corollary}

% Generated evidence asset; do not edit by hand.
\begin{figure}[p]\centering
\begin{tikzpicture}[x=9cm,y=6.8cm]
\draw[->] (0,0)--(1.07,0) node[right] {$\rho$}; \draw[->] (0,0)--(0,1.06) node[above] {discovery};
\draw[planner,very thick] (0,1)--(1,1) node[anchor=south east] {centralized $V_2=1$};
\draw[very thick,accent,dashed,mark=*] plot coordinates {(0.000000,0.840000) (0.166667,0.800000) (0.250000,0.780000) (0.500000,0.720000) (0.583333,0.700000) (0.750000,0.660000) (1.000000,0.600000)};
\draw[very thick,private,solid,mark=*] plot coordinates {(0.000000,0.840000) (0.166667,0.800000) (0.250000,0.780000) (0.500000,0.720000) (0.583333,0.700000) (0.750000,0.710526) (1.000000,0.720000)};
\draw[very thick,shared,solid,mark=*] plot coordinates {(0.000000,0.720000) (0.166667,0.700000) (0.250000,0.703125) (0.500000,0.710526) (0.583333,0.712500) (0.750000,0.715909) (1.000000,0.720000)};
\draw[dotted] (0.166667,0)--(0.166667,1.02) node[above,rotate=90,font=\scriptsize] {$1/6$};
\draw[dash dot] (0.53583372,0)--(0.53583372,1.02) node[above,rotate=90,font=\scriptsize] {$\rho^*$};
\draw[dotted] (0.583333,0)--(0.583333,1.02) node[above,rotate=90,font=\scriptsize] {$7/12$};
\draw[accent,dashed,very thick] (.68,.91)--(.76,.91) node[anchor=west] {direct private};
\draw[private,very thick] (.68,.85)--(.76,.85) node[anchor=west] {private selected};
\draw[shared,very thick] (.68,.79)--(.76,.79) node[anchor=west] {shared selected};
\draw (0,0)--(0,-.012) node[below] {0}; \draw (.25,0)--(.25,-.012) node[below] {$1/4$}; \draw (.5,0)--(.5,-.012) node[below] {$1/2$}; \draw (.75,0)--(.75,-.012) node[below] {$3/4$}; \draw (1,0)--(1,-.012) node[below] {1};
\end{tikzpicture}
\caption{Discovery versus source dependence at $p=3/5$. The exact theorem compares the private and shared selected equilibria, not direct clue-following. Selected sharing is strictly higher only on $(\rho^*,1)$, where $\rho^*=(5\sqrt{73}-17)/48$; equality returns at one. The selected shared outcome remains below centralized $V_2$. Lines connect registered exact cells for display; the interval theorem comes from the analytic certificate.}
\label{fig:strategic}
\ArtifactNote{Runs: \path{20260722T210334Z_DD-022_2376d5b7_ad67765ca8}; claims: DD-C-0104, DD-C-0105, DD-C-0106, DD-C-0110; generator: \path{distributed_discovery.papers.build_information_sharing_frontier}; input SHA-256 prefixes: \texttt{98fb7d475365}.}
\end{figure}
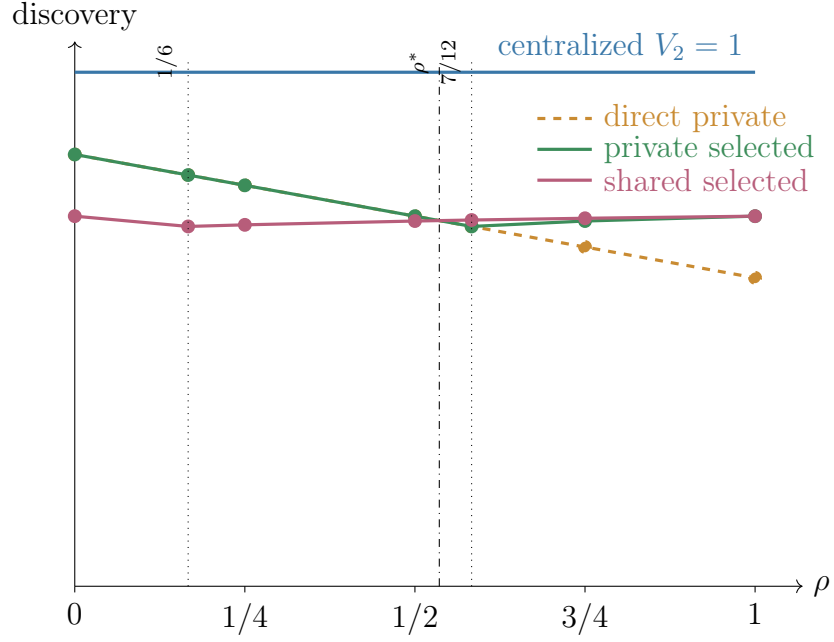

% Generated evidence table; do not edit by hand.
\begin{table}[t]\centering\scriptsize
\begin{tabularx}{\textwidth}{Ylll}\toprule object & exact value & display & interpretation\\\midrule
shared regime boundary & $1/6$ & 0.1667 & agreement rule changes\\
positive crossover $\rho^*$ & $\frac{5\sqrt{73}-17}{48}$ & 0.5358337235 & equality\\
private regime boundary & $7/12$ & 0.5833 & private anti-crowding begins\\
planner gap, low regime & $(7+3\rho)/25$ & positive & centralized $V_2=1$\\
planner gap, high regime & $(4+3\rho)/(13+12\rho)$ & positive & centralized $V_2=1$\\
\bottomrule\end{tabularx}
\caption{Exact canonical threshold, regime boundaries, and centralized implementation gap at $p=3/5$.}
\label{tab:threshold-gap}
\ArtifactNote{Runs: \path{20260722T210334Z_DD-022_2376d5b7_ad67765ca8}; claims: DD-C-0106, DD-C-0110; generator: \path{distributed_discovery.papers.build_information_sharing_frontier}; input SHA-256 prefixes: \texttt{146dcfba3f74, 8ae5af15a2f7}.}
\end{table}

The mechanism is not ``correlation is good.'' At low dependence, private clues
provide useful diversity and sharing induces avoidable crowding. On
$(\rho^*,7/12)$, private selected play still follows each clue fully; the gain
comes from sufficiently informative agreement posteriors under the selected
shared rule. On $(7/12,1)$, the private selected response also becomes
anti-crowding, reinforcing the selected sharing gain. At perfect dependence,
sharing adds no information relative to identical private clues, and the
selected values meet again. Dependence affects both posterior aggregation and
strategic duplication, but anti-crowding is not the mechanism throughout the
whole positive interval.

The bounded $7\times6$ registry contains 42 exact $(p,\rho)$ cells: six positive,
18 negative, and 18 neutral comparisons (DD-C-0108). This is a complete count of
the registered grid, not a measure of the interval's size and not evidence about
empirical frequencies.

% Generated evidence table; do not edit by hand.
\begin{table}[t]\centering\scriptsize
\begin{tabularx}{\textwidth}{Ylll}\toprule accuracy $p$ & positive & neutral & negative\\\midrule
$1/2$ & 0 & 7 & 0\\
$11/20$ & 4 & 1 & 2\\
$3/5$ & 2 & 1 & 4\\
$13/20$ & 0 & 1 & 6\\
$2/3$ & 0 & 1 & 6\\
$1$ & 0 & 7 & 0\\
\bottomrule\end{tabularx}
\caption{Exact DD-022 selected-sharing gain classes by signal accuracy. Each row contains the seven registered dependence values; counts are bounded cells, not empirical frequencies.}
\label{tab:strategic-gain-by-accuracy}
\ArtifactNote{Runs: \path{20260722T210334Z_DD-022_2376d5b7_ad67765ca8}; claims: DD-C-0108; generator: \path{distributed_discovery.papers.build_information_sharing_frontier}; input SHA-256 prefixes: \texttt{98fb7d475365}.}
\end{table}

The exact positive-interval theorem is intentionally fixed at $p=3/5$. On the
separate registered grid in \cref{tab:strategic-gain-by-accuracy}, positive cells
appear at $p=11/20$ and $p=3/5$ and have count zero at the higher listed
accuracies $p=13/20$ and $p=2/3$. This finite cross-tab describes the registered
cells only: it neither proves a monotone accuracy mechanism nor rules out a
positive interval outside the registered grid.

\section{Equilibrium selection and implementation failure}

The positive theorem compares two named equilibria. It does not state that
sharing improves every equilibrium or even the best or worst equilibrium in the
full correspondence. This qualification is mathematical, not rhetorical.

\begin{theorem}[Selection failure; DD-C-0109]
For every $\rho$ in the registered binary model, private play admits pure
equilibria in which the two agents choose opposite constant targets. These
equilibria attain discovery one. After sharing, ownership-aware disagreement
strategies can assign different targets when clues disagree and lie outside the
posterior-only identical-mixing selection. Therefore the strict interval in
\cref{eq:positive-interval} is not an every-equilibrium improvement result.
\end{theorem}

\begin{proof}[Argument]
With opposite constants, exactly one agent selects the true binary target in
every state. A unilateral deviation can only duplicate the other target in one
state or exchange which state is covered, so it does not yield a strict payoff
gain under the registered split-prize game. For shared disagreement histories,
the agents' ownership of the two private clues provides a role variable. Mapping
those roles to opposite actions covers both targets, but the mapping is not a
function solely of the common posterior. These constructions are checked
symbolically in the DD-022 proof record.
\end{proof}

% Generated evidence/selection asset; do not edit by hand.
\begin{figure}[t]\centering
{\renewcommand{\arraystretch}{1.3}\small\begin{tabularx}{\textwidth}{YYYY}\toprule
\textbf{Private alternative} & \textbf{Shared selection} & \textbf{Shared broader strategy} & \textbf{Centralized benchmark}\\\midrule
opposite constant targets & posterior-only identical mixing & ownership-aware disagreement split & binding top-two\\
$\Downarrow$ & $\Downarrow$ & $\Downarrow$ & $\Downarrow$\\
discovery one & selected positive interval; planner gap & disagreement coverage beyond posterior-only rules & $V_2=1$ by authority\\\bottomrule
\end{tabularx}}
\caption{Equilibrium and authority map. The posterior-only identical-mixing outcome is one selected decentralized equilibrium, not the correspondence. Constant-target private equilibria and ownership-aware disagreement strategies expose the selection boundary; centralized top-two uses binding authority.}
\label{fig:selection-map}
\ArtifactNote{Runs: \path{20260722T210334Z_DD-022_2376d5b7_ad67765ca8}; claims: DD-C-0106, DD-C-0109, DD-C-0110; generator: \path{distributed_discovery.papers.build_information_sharing_frontier}; input SHA-256 prefixes: \texttt{98fb7d475365}.}
\end{figure}

Equilibrium selection is a substantive institutional primitive. Anonymous
label-equivariant restrictions rule out fixed label roles. Posterior-only shared
strategies rule out conditioning on who observed which clue. Identical mixing
rules out asymmetric role assignment. Each restriction may be appropriate for
an application, but none follows merely from ``the agents share information.''
Classical equilibrium-selection theory similarly emphasizes that a game can
admit many self-consistent predictions \citep{HarsanyiSelten1988}.

\begin{theorem}[Centralized implementation gap; DD-C-0110]
At $p=3/5$, a centralized top-two planner attains $V_2=1$. Its gap above the
selected shared equilibrium is
\[
 \frac{7+3\rho}{25}\quad\text{in the low shared regime},
 \qquad
 \frac{4+3\rho}{13+12\rho}\quad\text{in the high shared regime},
\]
and is strictly positive.
\end{theorem}

Centralized top-two is an authority benchmark: it assigns distinct targets. It
is not an equilibrium generated by information sharing alone. Opposite constant
private equilibria also reach one, but for a different reason and without using
the clues. Reporting only the common value would hide the difference between
binding action assignment and decentralized strategic coordination.

Communication can also be strategic. Cheap-talk models show that what is sent
depends on preferences and incentives \citep{CrawfordSobel1982}; jury and voting
models show that information aggregation depends on strategic behavior
\citep{AustenSmithBanks1996,LevyRazin2015}. The current model holds truthful clue
sharing fixed and studies the action game that follows. It therefore isolates
one implementation boundary but does not solve strategic disclosure.

\Boundary{The positive sharing theorem, the selection failure, and the planner
gap are jointly necessary for interpretation. Dropping the selection qualifier
would change the theorem. Dropping the planner label would conflate information
with authority.}

\section{Design implications and limitations}

The results support a diagnostic sequence rather than a universal prescription.
First, specify the discovery objective: top-one accuracy, coverage by a portfolio,
or expected payoff can rank architectures differently. Second, measure signal
geometry through posterior coverage or residual error, not only marginal
accuracy. Third, state how many independent actions survive pooling. Fourth,
identify who can assign actions and which equilibrium or behavioral rule is being
compared. Finally, preserve null and negative findings within their registered
scope.

For centralized teams, the action-budget profile is the natural planning object.
If $V_1<P_N$ but $V_2\geq P_N$, the informational case for sharing may be sound
while a one-action implementation is not. The design response could be to reserve
a second action, maintain a private red team, or make the pooled output a
shortlist rather than a single recommendation. The theorem does not price these
options; it reveals where such pricing enters.

For decentralized systems, role design can matter as much as disclosure. Agent
identifiers, clue ownership, randomized conventions, and assignment authority
change the implementable strategy space. Information sharing can improve one
symmetric selection while an asymmetric or ownership-aware equilibrium performs
better. A deployment claim should therefore name its selection mechanism rather
than report ``the equilibrium.''

Social-learning models explain how repeated observation can create herding,
cascades, and overweighting of common sources
\citep{Banerjee1992,BikhchandaniEtAl1992,DeMarzoEtAl2003}. The current finite
models share the concern that correlation and duplication are easy to misread,
but they do not estimate a dynamic learning process. Public information can also
have coordination benefits or costs in macroeconomic games
\citep{MorrisShin2002,CornandHeinemann2008}; our discovery objective and rescue
portfolio produce a different, explicitly action-budgeted frontier.

Information design asks how a designer should choose signals or recommendations
\citep{KamenicaGentzkow2011,BergemannMorris2016}. We mostly take the registered
channels as given. The frontier could become a design constraint---choose a
pooled signal whose residual error contracts beyond the rescue threshold---but
optimizing over all feasible experiments is outside the verified claims. Team
theory likewise asks how information should be allocated under decentralized
decision rules \citep{Radner1962}. Our contribution is a small exact bridge from
pooled coverage to discovery portfolios and selection boundaries.

Several limitations are structural. Rescue actions are independent in the
nonstrategic identity; correlated rescue would require their joint failure
probability. The finite channel registries do not span arbitrary experiments.
The strategic model is binary, two-agent, and uses a specific split-prize payoff.
Sharing is truthful and costless. No agent can acquire additional information,
and no communication network is modeled. The planner values discovery only and
can bind actions. The strategic positive interval fixes $p=3/5$ and two named
equilibrium selections. No welfare claim is made beyond the registered discovery
and payoff objects.

The empirical boundary is equally important. No participants were recruited,
no human or real organizational data were collected, and no experiment was run.
The displayed counts are exact bounded enumerations of synthetic model inputs.
They should motivate experimental contrasts, not be described as observations.

% Generated evidence/authority asset; do not edit by hand.
\begin{figure}[t]\centering\small
\begin{tabularx}{\textwidth}{Ylll}\toprule layer & evidence & authority & boundary\\\midrule
Signal geometry & exact bounded & centralized $V_L$ & scalar insufficiency\\
Incremental sharing & identity/theorem & nonstrategic protocol & arbitrary-channel failure\\
General frontier & theorem plus census & centralized recovery & mixed-curve bounded null\\
Strategic sharing & selected BNE theorem & autonomous actions & not every equilibrium\\
Selection failure & verified negative & broader strategy space & discovery-one alternatives\\
\bottomrule\end{tabularx}
\caption{Evidence and authority map. Analytic, bounded exact, centralized, selected decentralized, and negative-boundary claims remain distinct even when they form one theorem family.}
\label{fig:evidence-map}
\ArtifactNote{Runs: \path{20260722T084145Z_DD-019_a77bb786_04a5e9f0c5, 20260722T142551Z_DD-020_3854fff6_37c11a850a, 20260722T185924Z_DD-021_3cdbbc40_2fea269a9a, 20260722T210334Z_DD-022_2376d5b7_ad67765ca8}; claims: DD-C-0089--DD-C-0110; generator: \path{distributed_discovery.papers.build_information_sharing_frontier}; input SHA-256 prefixes: \texttt{a7da8a20aa24, 607d03fba773, a4f63d52f402, 0ad5c3249e7d, 3521286a3415, 2bd510abf8cd, e897c62593b4, 98fb7d475365, 8ae5af15a2f7, 146dcfba3f74}.}
\end{figure}

% Generated evidence table; do not edit by hand.
\begin{table}[t]\centering\scriptsize
\begin{tabularx}{\textwidth}{Ylll}\toprule claim range & owner & evidence class & manuscript role\\\midrule
DD-C-0089--0091 & DD-019 & independent repository method & geometry and recovery\\
DD-C-0092--0094 & DD-020 & identity and verified theorems & aggregation versus rescue\\
DD-C-0095--0096 & DD-020 & independent repository method & census and counterchannel\\
DD-C-0097--0098 & DD-021 & verified theorems & frontier and centralized recovery\\
DD-C-0099--0102 & DD-021 & independent repository method & classes, witnesses, budgets\\
DD-C-0103 & DD-021 & verified bounded negative & mixed-curve null\\
DD-C-0104--0107 & DD-022 & verified theorem/corollary & selected equilibria and interval\\
DD-C-0108 & DD-022 & independent repository method & 42-cell classification\\
DD-C-0109 & DD-022 & verified negative & selection failure\\
DD-C-0110 & DD-022 & verified theorem & implementation gap\\
\bottomrule\end{tabularx}
\caption{Claim, evidence, and ownership map. Here independently implemented repository verification is not external replication. Status belongs to each claim, not to the paper as a whole.}
\label{tab:claim-map}
\ArtifactNote{Runs: \path{20260722T084145Z_DD-019_a77bb786_04a5e9f0c5, 20260722T142551Z_DD-020_3854fff6_37c11a850a, 20260722T185924Z_DD-021_3cdbbc40_2fea269a9a, 20260722T210334Z_DD-022_2376d5b7_ad67765ca8}; claims: DD-C-0089--DD-C-0110; generator: \path{distributed_discovery.papers.build_information_sharing_frontier}; input SHA-256 prefixes: \texttt{a7da8a20aa24, 607d03fba773, a4f63d52f402, 0ad5c3249e7d, 3521286a3415, 2bd510abf8cd, e897c62593b4, 98fb7d475365, 8ae5af15a2f7, 146dcfba3f74}.}
\end{table}

The evidence architecture makes extension possible. A new channel family could
be registered and classified by its contraction ratios. A richer strategic
study could expand the equilibrium correspondence or introduce an explicit
selection institution. An empirical study could preregister sharing and action-
budget treatments. None of those future directions is silently included here.

\section{Conclusion}

Information sharing improves decentralized discovery when its aggregation gain
is large enough to replace the independent rescue capacity it removes and when
the relevant decentralized behavior actually implements that gain. The first
condition is captured exactly by the residual-error frontier
$e_{s+1}/e_s<1-q$. The second requires an equilibrium or institutional account.

Finite channel examples show why accuracy alone cannot answer the question.
Equal one-person accuracy and equal private discovery can coexist with sharply
different pooled profiles, recovery budgets, and sharing directions. Full
centralized action capacity always recovers the private baseline, but the
required budget is part of the design, not a free consequence of communication.

The strategic model supplies a strict positive interval without planner action
assignment, but it simultaneously supplies the boundary: the result belongs to
named equilibrium selections. Alternative private and ownership-aware strategies
and the centralized top-two benchmark prevent an every-equilibrium or pure
information interpretation. The useful conclusion is therefore conditional and
operational. Share when pooled residual error clears the rescue threshold, retain
enough action capacity, and state how diverse actions will be selected.

\appendix
\section{Proof details for the sharing identities}

This appendix records algebra used by the synthesis; the scientific ownership
remains with DD-020 and DD-021. Let $F_s$ denote failure of the pooled block and
$R_i$ failure of private rescue action $i$. Under the registered protocol,
$\Prb(F_s)=1-C_s$ and the $N-s$ rescue failures are independent with probability
$1-q$. Therefore
\[
 \Prb(F_s\cap R_1\cap\cdots\cap R_{N-s})=(1-C_s)(1-q)^{N-s}.
\]
Taking one minus this probability proves \cref{eq:incremental-value}. At the next
step, one fewer rescue action remains. Subtraction gives
\begin{align*}
G_{s+1}-G_s
&=(1-C_s)(1-q)^{N-s}-(1-C_{s+1})(1-q)^{N-s-1}\\
&=(1-q)^{N-s-1}\left[(1-q)e_s-e_{s+1}\right].
\end{align*}
The factor outside the brackets is nonnegative and is positive away from the
degenerate $q=1$ boundary. If $e_s>0$, division by $e_s$ produces the contraction
ratio. Notice that $C_{s+1}\geq C_s$ alone is not enough. The absolute gain must
be large enough relative to $q e_s$, the rescue value applied to the current
pooled miss probability.

Two boundary cases deserve explicit labels. If $e_s=0$, the current pooled block
already discovers surely; the ratio is undefined but the undivided identity
remains valid. If $q=1$, any remaining rescue action guarantees discovery, so
absorbing the last perfect rescuer cannot improve the portfolio. These cases do
not affect the registered theorem rows but prevent careless division.

\section{Canonical exact records}

For auditability, the principal frozen profiles are restated here. DD-019 records
\begin{align*}
\text{noisy point: }&(7/12,43/54,25/27),&P_3&=7/8,&L^*&=3,\\
\text{noisy shortlist: }&(35/64,79/96,539/576),&P_3&=387/512,&L^*&=2,\\
\text{guaranteed shortlist: }&(17/18,1,1),&P_3&=7/8,&L^*&=1,\\
\text{exclusion: }&(16/27,26/27,1),&P_3&=19/27,&L^*&=2,\\
\text{confidence point: }&(295/384,175/192,31/32),&P_3&=485/512,&L^*&=3.
\end{align*}
The associated DD-020 paths are respectively
\begin{align*}
&(7/8,3/4,7/12),\quad (387/512,11/16,35/64),\\
&(7/8,11/12,17/18),\quad (19/27,17/27,16/27),\\
&(485/512,113/128,295/384).
\end{align*}
These are exact source values, not numbers reconstructed from the plotted lines.

At $p=3/5$, DD-022 records the selected discovery comparisons at
$\rho\in\{0,1/6,1/4,1/2,7/12,3/4,1\}$. The private selected sequence is
$21/25,4/5,39/50,18/25,7/10,27/38,18/25$; the shared selected sequence is
$18/25,7/10,45/64,27/38,57/80,63/88,18/25$. The source CSV distributed with
\cref{fig:strategic} includes these cells and the centralized value one.

\section{Protocol-level derivations and interpretation checks}

This appendix makes explicit the sequence of objects that can otherwise be
compressed into the phrase ``sharing improves discovery.'' It introduces no new
claim; it is an editorial derivation and interpretation audit of DD-019 through
DD-022.

\subsection*{From a channel to an action-budget profile}

Fix a prior and a finite signal history $h$. Bayes' rule produces posterior
masses $\pi_j(h)$ over candidates. Sort them as
$\pi_{(1)}(h)\geq\cdots\geq\pi_{(M)}(h)$. Conditional on $h$, a centralized
planner with $L$ distinct actions covers the correct candidate with probability
$\sum_{j=1}^L\pi_{(j)}(h)$. Averaging over histories produces $V_L$. This
calculation uses the complete posterior vector. Keeping only
$\max_j\pi_j(h)$ would determine $V_1$ but discard the marginal value of actions
two through $M$.

That lost information explains the equal-accuracy example. Accuracy constrains
the expected success of a registered top action. It does not determine whether
the posterior error is concentrated on one runner-up or spread over several
candidates. A guaranteed shortlist can put essentially all mass into a known
small set even when its top-label correctness is only one half. A noisy point can
leave substantial mass outside the first two ranks. Because discovery is a
coverage objective, the ordering of the remaining masses matters directly.

The profile is centralized by construction. A decentralized team could attain
the same set of actions if it had an implementation device: binding assignments,
public roles, or a supported equilibrium that maps agents to distinct ranks.
Absent such a device, $V_L$ remains a benchmark. This label prevents a common
category error in which a planner's ordered posterior sum is reported as the
automatic consequence of agents observing the same evidence.

\subsection*{From a profile to an incremental-sharing path}

The registered incremental protocol chooses a block size $s$. Its pooled block
has a specified coverage $C_s$; the remaining agents act as independent rescue
attempts. The comparison between $s$ and $s+1$ therefore changes both the pooled
history and the number of rescue attempts. Holding the latter fixed would answer
a different question: the marginal value of an additional signal with no
capacity cost. Holding the pooled coverage fixed would answer the marginal value
of an independent attempt. The sharing frontier is the net comparison.

The failure representation is useful because discovery is a union event. The
current architecture fails only if the pooled block fails and every rescue
attempt fails. The next architecture may have a smaller pooled failure
probability, but it also contains one fewer multiplicative rescue factor. The
break-even equation
\[
 e_{s+1}=(1-q)e_s
\]
says that the pooled block must remove the same share of its residual error as a
private attempt would have removed. A strict inequality determines the sign
without requiring the absolute level of $G_s$.

The ratio formulation is scale-free but not assumption-free. If rescue attempts
are correlated, $(1-q)^{N-s}$ must be replaced by their joint conditional failure
probability. If rescue quality depends on the pooled history, the common $q$
must be indexed by that history. If pooled agents retain individual actions, the
capacity transition is no longer one-for-one. Each extension is feasible, but
none is silently covered by DD-C-0097.

\subsection*{From a theorem to a bounded classification}

The analytic frontier gives a predicate for each applicable step. DD-021 then
applies that predicate to a versioned registry. A curve is called aggregation
dominated when each strict step lies below the rescue threshold, compression
dominated when each strict step lies above it, all-neutral when equality holds
throughout, and mixed when different steps have different strict signs. The
class labels are summaries of exact records, not primitives in the theorem.

The distinction between proof and census is important. A test that recomputes
all 177 rows can detect implementation drift and show that the theorem agrees
with those fixtures. It cannot prove the formula for unenumerated parameters.
Conversely, the algebraic proof does not imply that the registry contains a
mixed witness. Keeping the analytic and bounded layers separate is what allows
the zero mixed count to remain a useful negative result without becoming an
unsupported universal monotonicity claim.

Minimal witnesses are also registry-relative. The word ``minimal'' refers to
the frozen lexicographic ordering over registered parameters. It does not mean
minimal among every mathematically possible channel or in description length.
This qualifier is carried into \cref{tab:witnesses} and the generated asset note.

\subsection*{From private clues to selected equilibria}

The DD-022 comparison begins with a hidden mixture over source dependence. The
parameter $\rho$ is the prior probability of the common-source branch. Observing
two agreeing clues raises confidence in their label and changes beliefs about
which branch was likely, but it does not reveal the branch realization. A
strategy that conditions on ``the clues were common'' would therefore use
information unavailable to the agents.

Before sharing, each agent sees one clue and not the other's. Anonymous
label-equivariance requires the same response after relabeling targets. The
selected symmetric equilibrium can then be summarized by a clue-following
probability $r$. The statistic $A=t^2+\rho(1-t^2)$ captures the unconditional
alignment created by clue accuracy and source dependence. At $A=3t$, the
boundary solution gives way to interior anti-crowding. The exact piecewise
formula in \cref{tab:equilibrium-formulas} is the registered selection.

After sharing, the information histories are agreement on a label or
disagreement. At agreement, $u$ is the posterior probability that the agreed
label is correct. Identical posterior-only mixing is summarized by $x$. At
disagreement the posterior over targets is one half, so the selected identical
mix is one half. This last step is where ownership information is deliberately
discarded: although the common posterior is symmetric, agent one and agent two
can still know which private clue each originally owned.

The positive interval compares the discovery induced by these two selections.
It is meaningful because both are exact Bayesian Nash equilibria under their
registered restrictions. It is limited because equilibrium existence does not
select them from all alternatives. An opposite constant-target profile ignores
the signals but covers the binary state space. An ownership-aware rule can use
agent identity to split actions after disagreement. These constructions do not
invalidate the selected interval; they invalidate a stronger interpretation of
it.

\subsection*{Three meanings of discovery one}

The value one appears in three places that must not be collapsed. First, a
centralized planner with two actions covers both binary targets by binding
assignment. Second, opposite constant-target private equilibria cover both
targets through a fixed decentralized convention. Third, an ownership-aware
shared rule can cover both targets on disagreement histories while behaving
differently on agreement. The same discovery number masks different evidence
use, authority, and off-path behavior.

For organizational design, this means that performance tables need an authority
column. A benchmark labeled ``centralized top-two'' is evidence about attainable
coverage under control, not about spontaneous coordination. A row labeled
``selected shared equilibrium'' is evidence about one decentralized behavioral
rule. A row labeled ``alternative equilibrium'' is evidence about multiplicity.
The generated selection map and evidence-authority map enforce those labels
visually.

\subsection*{Reading the canonical plot}

In \cref{fig:strategic}, the dashed direct-private curve is not the private
selected equilibrium after the private anti-crowding boundary. The selected
private curve initially follows direct clue use, then changes when mixing becomes
optimal. The selected shared curve changes formula at its own boundary $1/6$.
Its intersection with the private selected curve at $\rho^*$ is algebraic, not a
linear interpolation of plotted cells. The second equality at $\rho=1$ reflects
the fact that perfectly common clues contain no additional cross-agent evidence
when shared.

The connected points in the figure are a display of registered exact cells. The
theorem between them comes from piecewise symbolic formulas and a certified root
isolation. The decimal marker is only a visual aid. The exact expression and
rational bracket control every sign claim.

\subsection*{Interpretation checklist}

For each proposed use of the results, the following questions identify the
relevant row of the evidence map:
\begin{enumerate}[leftmargin=2em]
\item Is the outcome top-choice accuracy or coverage by at least one action?
\item Does sharing reduce the number of independent actions, and by how much?
\item Are remaining rescue attempts independent conditional on pooled failure?
\item Is the displayed action budget feasible and centrally assignable?
\item If actions are decentralized, which equilibrium selection or role system
      maps shared evidence to distinct actions?
\item Is a reported count an analytic theorem, a bounded exact census, or a
      preserved null inside a registered search space?
\item Does the protocol reveal only the pooled clues, or also an unavailable
      latent source branch?
\item Are claims about model-generated values kept separate from claims about
      people, organizations, or deployed multi-agent systems?
\end{enumerate}
An application that cannot answer these questions has not yet specified the
comparison made by this paper.

\section{Evidence, provenance, and reproducibility}

The generator resolves all 22 claim IDs against the claim ledger, verifies the
ten source-output hashes against immutable run manifests, writes exact CSV data
for every figure, and writes TeX tables and figures with run IDs, claim IDs, the
generator name, and input checksum prefixes. It then builds the PDF twice under
a fixed source epoch and requires byte-identical hashes. An independent paper
audit rechecks the source contracts without importing the generator and performs
deliberate corruption tests on copied metadata.

Here ``independent'' means independently implemented repository verification:
the second method or audit does not import or call the primary generator path.
It is not external replication by an unaffiliated team. The complete immutable
records, independent verifier outputs, deliberate corruption checks, and paper
generator are public at repository commit
\nolinkurl{66ba4449572b20c519a4629ef20cfc030fae1762}. The four run directories and
their complete manifest SHA-256 digests are:
\begin{itemize}[leftmargin=1.4em]
\item \url{https://github.com/yoheinakajima/distributed-discovery/tree/66ba4449572b20c519a4629ef20cfc030fae1762/results/verified/20260722T084145Z_DD-019_a77bb786_04a5e9f0c5},
  \nolinkurl{011c7699e6bf2bb90b1201797a430b88b74228385e74df61d23d03a124670c53};
\item \url{https://github.com/yoheinakajima/distributed-discovery/tree/66ba4449572b20c519a4629ef20cfc030fae1762/results/verified/20260722T142551Z_DD-020_3854fff6_37c11a850a},
  \nolinkurl{82f8c71d65b27a48b075d06d0a186040cc9e6015aaee9bc695ff2e9fa823a972};
\item \url{https://github.com/yoheinakajima/distributed-discovery/tree/66ba4449572b20c519a4629ef20cfc030fae1762/results/verified/20260722T185924Z_DD-021_3cdbbc40_2fea269a9a},
  \nolinkurl{69ac4c760d7a34b6ee315ae619287b7928f69cc8a20eaf4a62202868baadcacd};
\item \url{https://github.com/yoheinakajima/distributed-discovery/tree/66ba4449572b20c519a4629ef20cfc030fae1762/results/verified/20260722T210334Z_DD-022_2376d5b7_ad67765ca8},
  \nolinkurl{3d905723eda6069cf870f1fb7b2a241cfa592ea43e9009473bd2b0878bdd5bb1}.
\end{itemize}
The principal review-facing outputs are DD-020
\path{outputs/channel-profiles.json} at
\nolinkurl{a4f63d52f402ce14f91e469d4b425ce65a92740b31530c7da9384cc726030e3a},
DD-021 \path{outputs/registry.json} at
\nolinkurl{3521286a341595d27b95248a0720f4df672bb1b811d17edf7f286b5520912697},
and DD-022 \path{outputs/registry.json} at
\nolinkurl{98fb7d475365548fcc723603b3c3721e673b0b0423bc89579c411c3252ac54f1}.
Their run directories also contain the full verification and corruption-test
records. The paper generator is
\url{https://github.com/yoheinakajima/distributed-discovery/blob/66ba4449572b20c519a4629ef20cfc030fae1762/src/distributed_discovery/papers/build_information_sharing_frontier.py};
the arXiv source package includes compact copies of the cited immutable evidence.

The manuscript makes no new claim record and no new study or run. DD-019 owns
the signal-geometry results, DD-020 the incremental identity and counterchannel,
DD-021 the general frontier and recovery registry, and DD-022 the strategic
equilibrium results. Editorial synthesis does not change those ownership lines.
The citation database is copied into the generated bundle at build time so the
compiled bibliography is pinned to the paper artifact.

\section{Literature boundary}

The adjacent literatures establish why narrow contribution language is needed.
Blackwell comparison concerns informativeness across decision problems
\citep{Blackwell1953}. Team decision theory concerns decentralized information
allocation \citep{Radner1962}. Herding and social-learning models concern the
endogenous informational content of observed behavior
\citep{Banerjee1992,BikhchandaniEtAl1992,AcemogluEtAl2011}. Strategic
communication concerns incentive-compatible messages \citep{CrawfordSobel1982}.
Information design concerns signal or recommendation choice
\citep{KamenicaGentzkow2011,BergemannMorris2016}. Equilibrium selection concerns
which equilibrium is predicted when several exist \citep{HarsanyiSelten1988}.

The present synthesis composes a particular portfolio objective, rescue
technology, coverage profile, and selected action game. That composition may be
useful without supporting a claim that no related frontier, anti-crowding
mechanism, or selection caveat exists elsewhere. The accompanying literature
files record stable identifiers and the claim-by-claim relationship.

\bibliographystyle{plainnat}
\bibliography{generated/references}
\end{document}